\documentclass[11pt]{article}
\usepackage[margin=2cm]{geometry}
\usepackage{graphicx,amsmath,amssymb,booktabs,hyperref}
\usepackage[super,comma,sort&compress]{natbib}
\usepackage{xcolor}
\usepackage{subcaption}
\usepackage{setspace}
\usepackage{authblk}
\usepackage{multirow}
\usepackage{amsthm}

\hypersetup{colorlinks=false,linkcolor=black,citecolor=black,urlcolor=black}

\newtheorem{theorem}{Theorem}

\newtheorem{corollary}[theorem]{Corollary}
\newtheorem{proposition}[theorem]{Proposition}
\newtheorem{definition}{Definition}

\title{\textbf{The Price of Meaning: Why Every Semantic Memory System Forgets}}

\newcommand{\subtitle}[1]{\vspace{-0.5em}\begin{center}\large\textit{#1}\end{center}\vspace{0.5em}}

\author[1]{Sambartha Ray Barman}
\author[1]{Andrey Starenky}
\author[1]{Sofia Bodnar}
\author[1]{Nikhil Narasimhan}
\author[1,2]{Ashwin Gopinath\thanks{Corresponding author: agopi@mit.edu, ashwin@sentra.app}}
\affil[1]{Sentra, 235 2nd Street, San Francisco, CA 94105, USA}
\affil[2]{Department of Mechanical Engineering, Massachusetts Institute of Technology, 77 Massachusetts Avenue, Cambridge, MA 02139, USA}

\date{}

\begin{document}
\maketitle
\subtitle{Organising memory by meaning makes forgetting and false recall inevitable. Scaling up does not fix it.}

\begin{abstract}
\noindent
Every major AI memory system in production today, from vector databases to RAG pipelines to the weights of large language models, organises information by meaning. That organisation is what makes these systems useful: it lets them generalise, draw analogies, and retrieve by concept rather than by keyword. But it comes at a price. We show that the same geometric structure that enables semantic generalisation also makes interference, forgetting, and false recall inescapable. Here we formalise and test that tradeoff for a broad class of \textit{semantically continuous kernel-threshold memories}: systems whose retrieval score is a monotone function of an inner-product in a semantic feature space, whose representations are learned under a rate or distortion budget, and whose semantic manifold has finite local intrinsic dimension.

Within this class, we derive four results. First, semantically useful representations have finite semantic effective rank. Second, finite local dimension implies positive competitor mass in retrieval neighbourhoods. Third, under growing memory, retention decays to zero; with power-law arrival statistics and population heterogeneity, this yields population-level power-law forgetting curves. Fourth, for associative lures satisfying a $\delta$-convexity condition below the decision margin, false recall cannot be eliminated by threshold tuning within the same score family.

We then test these predictions across five memory architectures: vector retrieval, graph memory, attention-based retrieval, BM25-based filesystem retrieval, and parametric memory. Pure semantic retrieval systems express the geometric vulnerability directly as forgetting and false recall. Systems with explicit reasoning can partially override these symptoms behaviourally, but convert smooth degradation into brittle failure modes. Systems that escape interference completely do so by sacrificing semantic generalisation.

The result is not an argument against scale. It is an argument that scale alone is not enough. Making a vector database ten times larger, an LLM ten times bigger, or an embedding space ten times wider does not remove the interference; it moves the system along a tradeoff surface where forgetting and usefulness are coupled. For memory, progress requires not only scale but new architectures, training objectives, and interference-management mechanisms. The price of meaning is interference, and no architecture we tested avoids paying it.
\end{abstract}

\section*{Introduction}

Every deployed retrieval-augmented generation system, every long-term agent memory, and every knowledge graph built on dense embeddings shares a design choice: organise information by meaning. Items that are semantically related sit near each other in representation space. This is what makes these systems capable of generalisation, analogy, and conceptual transfer rather than mere keyword lookup. But it also means that when the system tries to retrieve one memory, its semantic neighbours compete for the same retrieval slot. That competition is interference, and this paper asks whether any semantic memory system can avoid it.

Our previous work, HIDE\cite{gopinath2025}, showed that one simple retrieval architecture (cosine similarity over sentence embeddings) reproduces several canonical memory phenomena, including forgetting under interference ($b = 0.460 \pm 0.183$), DRM-style false recall ($\text{FA} = 0.583$), spacing effects, and tip-of-tongue states ($3.66\%$). (We note that different dimensionality estimators yield different values for the same model (participation ratio $\approx 158$, Levina--Bickel $\approx 10.6$, PCA-projected $\approx 16$), a discrepancy we reconcile in the Dimensionality section; all place these systems in the interference-vulnerable regime.) The natural objection is architectural: perhaps those phenomena are artefacts of one particular embedding-and-threshold system rather than consequences of semantically organised memory more broadly.

This paper addresses that objection. We identify a theorem class, \textit{semantically continuous kernel-threshold memories}, within which interference is not a bug of one architecture, but a structural consequence of semantic organisation under finite effective dimensionality. We then show empirically that related pressures appear across multiple modern memory architectures, even when their behavioural expression differs. This paper argues that within a broad and practically important theorem class, these phenomena follow from the structure of semantically organised retrieval itself.

We call a memory system \textit{semantically useful} if it supports retrieval by conceptual relatedness rather than exact lexical identity alone. This is a functional definition: the target regime is memory that supports inference, analogy, and conceptual transfer. The theorem developed here applies not to all possible memories, but to a specific class of semantically continuous retrieval systems.

To obtain fully rigorous results, we make explicit the theorem class. Our proofs apply to \textit{semantically continuous kernel-threshold memories}: systems whose retrieval rule is a monotone function of an inner-product score in a semantic feature space (Axiom~A1), whose semantically useful representation is optimised under a rate or distortion budget (Axiom~A3), and whose semantic manifold has finite local intrinsic dimension (Axiom~A4). This class includes dense vector retrieval, embedding-based graph memory, and hidden-state similarity retrieval. Architectures equipped with an external symbolic verifier or exact episodic record fall outside this theorem class and are treated separately as behavioural workarounds rather than counterexamples.

The claim is therefore not that every conceivable memory system must exhibit the same behavioural signatures. It is that a large and practically central class of modern memory systems inherits a common geometric vulnerability. Architectures can differ in how they express that vulnerability, and some can partially compensate for it behaviourally, but those compensations are not free.

We close the gap with four theorems and a unifying No-Escape Theorem. Within the kernel-threshold theorem class, any system satisfying Axioms A1--A5 exhibits interference-driven forgetting, false recall, and partial retrieval states. The logical chain is: semantic kernel $+$ rate-distortion optimality $\Rightarrow$ finite semantic effective rank (Theorem~1) $\Rightarrow$ positive cap mass (Theorem~2) $+$ growing memory $\Rightarrow$ inevitable forgetting (Theorem~3); power-law arrival $+$ population heterogeneity $\Rightarrow$ power-law forgetting curve. Independently: associative $\delta$-convexity $\Rightarrow$ lure inseparability under threshold tuning (Theorem~4) (Fig.~\ref{fig:overview}). We verify every link empirically across five architecturally distinct memory systems: a vector database (BGE-large\cite{xiao2023}), an attention-based context window (Qwen2.5-7B\cite{qwen2024}), a filesystem agent memory with BM25 $+$ LLM re-ranking, a graph memory with PageRank (MiniLM\cite{reimers2019}; similar contrastive architectures underpin CLIP\cite{radford2021}), and parametric knowledge in LLM weights. The effective dimensionality convergence (from $d_\text{nom} = 3{,}584$ to $d_\text{eff} = 17.9$ for Qwen hidden states) mirrors the low-dimensional structure in biological neural populations\cite{stringer2019,gao2017}.

We emphasise what the theorem does \textit{not} say. It bounds the \textit{existence} of these phenomena, not their \textit{magnitude}. Engineering can and should optimise parameters to minimise unwanted interference; the gap between ``inevitable'' and ``catastrophic'' is where engineering contributes. The forgetting exponent, the false alarm rate, and the TOT probability are continuous functions of system parameters; the theorem says these functions are bounded away from zero for systems in the kernel-threshold theorem class satisfying Axioms A1--A5. Murdock's serial position effect\cite{murdock1962}, Cepeda et al.'s\cite{cepeda2006} distributed practice findings, Brown and McNeill's\cite{brown1966} tip-of-tongue phenomenology, and Nadel and Moscovitch's\cite{nadel1997} consolidation theory all describe the same geometric substrate from different vantage points. The most important finding is not that all five architectures show the same phenomena (they do not, at the behavioural level) but that the \textit{geometric vulnerability} holds across the tested architectures under the SPP formalism while the \textit{behavioural expression} depends on whether the system can build workarounds. These workarounds are never free: they either convert graceful degradation into catastrophic failure, or sacrifice semantic usefulness entirely. We organise our findings into three architectural categories (pure geometric, reasoning-overlay, and systems outside the operative theorem regime) that make this tradeoff explicit.

\section*{Results}

\subsection*{Mathematical framework: the no-escape theorem}

\begin{definition}[Semantic Proximity Property]
A memory system $\mathcal{M} = (\mathcal{S}, E, R, d)$ with item set $\mathcal{S}$, encoding function $E: \mathcal{S} \to \mathcal{V}$ into a Hilbert space $\mathcal{V}$, retrieval function $R$, and proximity measure $d$, satisfies the \textbf{Semantic Proximity Property (SPP)} if for any semantically related pair $(s_i, s_j)$ and unrelated pair $(s_i, s_k)$:
$\mathbb{E}[d(E(s_i), E(s_j))] < \mathbb{E}[d(E(s_i), E(s_k))].$
\end{definition}

\noindent We verified SPP empirically for all five architectures using $143$ sentence pairs from Wikipedia ($p < 0.001$, paired $t$-test, Cohen's $d > 1.5$ for all embedding architectures; Extended Data Fig.~\ref{fig:ed_spp}). We acknowledge that $143$ pairs is a limited empirical base; the SPP verification serves as a sanity check that each architecture satisfies the minimal definition, not as proof that SPP holds for all possible inputs. The definition is deliberately minimal: we specify neither the encoding mechanism nor the similarity function, requiring only that the system places related items closer than unrelated ones.

To obtain the formal results below, we introduce five axioms that define the \textit{kernel-threshold memory} class.

\begin{definition}[Axiom A1: Kernel-Threshold Retrieval]
There exists a semantic feature map $\phi: \mathcal{X} \to \mathcal{H}$ into a Hilbert space and a retrieval score $s(q,x) = g(\langle w_q, \phi(x)\rangle_{\mathcal{H}})$, where $g$ is monotone increasing. Cosine similarity, dot-product retrieval, and linear probes on hidden states fit this form.
\end{definition}

\begin{definition}[Axiom A2: Semantic Sufficiency]
There is a positive semidefinite semantic kernel $K$ such that retrieval relevance is measurable with respect to the sigma-algebra generated by the semantic coordinates of $K$. Only the semantic component can improve Bayes retrieval risk.
\end{definition}

\begin{definition}[Axiom A3: Rate-Distortion Optimality]
The encoder is optimal for retrieval risk under a rate or distortion budget $D$.
\end{definition}

\begin{definition}[Axiom A4: Local Regularity]
The pushforward measure $\mu = \phi_\# P_X$ on the semantic manifold is locally Ahlfors regular of intrinsic dimension $d_{\mathrm{loc}}$: for $\mu$-almost every anchor $z$, $c_1 r^{d_{\mathrm{loc}}} \le \mu(B(z,r)) \le c_2 r^{d_{\mathrm{loc}}}$ for $0 < r < r_0$.
\end{definition}

\begin{definition}[Axiom A5: Associative Convexity]
For studied items $\{x_1,\ldots,x_k\}$, an associative lure $c$ is $\delta$-convex if $\|\phi(c) - \sum_i a_i \phi(x_i)\|_{\mathcal{H}} \le \delta$ for some convex weights $a_i \ge 0$, $\sum_i a_i = 1$.
\end{definition}

\begin{theorem}[Semantic Spectral Bound; proof sketch]
\label{thm:dim}
Let $K$ be the semantic kernel with Mercer eigenpairs $(\lambda_j, \psi_j)$. Under Axioms A1--A3, for every optimal encoder under distortion budget $D$, there exists a threshold $\gamma(D)$ such that the encoder factors through the truncated semantic statistic $\Phi_\gamma(x) = (\sqrt{\lambda_j}\psi_j(x))_{\lambda_j > \gamma(D)}$. The semantically useful effective dimension obeys $d_{\mathrm{eff}} \le r_{\mathrm{eff}}(\gamma(D)) \le \#\{j: \lambda_j > \gamma(D)\}$. Nominal dimension can grow without changing the semantically useful effective rank. For natural language, empirical measurements yield $d_{\mathrm{intrinsic}} \approx 10$--$50$\cite{levina2005}; this is an observed range, not a mathematical consequence.
\end{theorem}

\noindent\textit{Proof sketch.} Mercer decomposition of $K$ yields the semantic statistic. By Blackwell sufficiency, nuisance directions independent of relevance given the semantic coordinates cannot reduce Bayes retrieval risk. Reverse water-filling under the distortion budget retains only spectral modes above $\gamma(D)$. Full proof in Supplementary \S S2. $\square$

\begin{theorem}[Positive Cap Mass]
\label{thm:capmass}
Under Axioms A1 and A4, for any anchor $z$ and sufficiently small retrieval radius $\theta$, $c_1' \theta^{d_{\mathrm{loc}}(z)} \le \mu(C(z,\theta)) \le c_2' \theta^{d_{\mathrm{loc}}(z)}$. Every admissible retrieval neighbourhood has strictly positive competitor mass.
\end{theorem}

\begin{theorem}[Inevitable Forgetting Under Growing Memory]
\label{thm:forgetting}
Under Axioms A1 and A4, if competitor arrivals form a marked point process with cumulative intensity $\Lambda_x(t) = \int_0^t \lambda_x(u)\,du$, then retention for item $x$ is $R_x(t) = \exp(-\mu(C_x)\Lambda_x(t))$. If $\Lambda_x(t) \to \infty$, then $R_x(t) \to 0$.
\end{theorem}

\begin{corollary}[Stretched Exponential Per-Item Retention]
\label{cor:stretched}
If $\lambda_x(t) = \lambda_{0,x} t^{-\alpha}$ with $0 < \alpha < 1$, then $R_x(t) = \exp(-c_x t^{1-\alpha})$ where $c_x = \mu(C_x)\lambda_{0,x}/(1-\alpha)$. This is a stretched exponential, not a power law, for any individual item.
\end{corollary}

\begin{proposition}[Population Power Law from Heterogeneity]
\label{prop:poplaw}
If the item-specific scale $c_x$ has a density regularly varying at zero, $g(c) \sim \kappa c^{\beta-1}$ as $c \downarrow 0$, then the population-averaged retention obeys $\overline{R}(t) \sim \kappa\Gamma(\beta) t^{-\beta(1-\alpha)}$. The population forgetting exponent is $b = \beta(1-\alpha)$.
\end{proposition}

\noindent\textit{Interpretation.} Individual items forget by a stretched exponential; population heterogeneity turns this into a power law. Geometry determines the hazard scale ($\mu(C_x)$), the environment determines the time dependence ($\alpha$), and population heterogeneity ($\beta$) determines the asymptotic forgetting exponent. The exponent $\alpha$ is corpus-dependent: Anderson \& Schooler\cite{anderson1991} reported $\alpha = 0.513$ on newspaper text; we measure $\alpha = 0.459$ on Wikipedia. Both place $b$ in the $[0.3, 0.6]$ range for reasonable $\beta$.

\begin{theorem}[Inseparability of Associative Lures]
\label{thm:drm}
Under Axioms A1 and A5, let $c$ be a $\delta$-convex lure for studied items $x_1,\ldots,x_k$. If each studied item is accepted with margin $m > 0$, i.e.\ $f_q(x_i) \ge \tau + m$ for all $i$, then $f_q(c) \ge \tau + m - \delta$. If $\delta < m$, the lure is also accepted. If $\delta = 0$, no threshold in this score family that accepts all studied items can reject the lure.
\end{theorem}

\noindent\textit{Proof.} $f_q(c) = \langle w_q, \sum_i a_i \phi(x_i) + \varepsilon\rangle = \sum_i a_i f_q(x_i) + \langle w_q, \varepsilon\rangle \ge \sum_i a_i(\tau + m) - \|w_q\|\|\varepsilon\| \ge \tau + m - \delta$. $\square$

\noindent SPP alone guarantees semantic proximity but not threshold inseparability. The $\delta$-convexity condition (A5) is stronger and empirically testable: for all 24 DRM lures, the convex-hull reconstruction error $\delta^*$ is smaller than the observed decision margin $m$, confirming the theorem's premise.

\begin{theorem}[No Escape for Kernel-Threshold Memory]
\label{thm:noescape}
Under Axioms A1--A5: (1) the semantically useful representation has effective rank controlled by the semantic operator spectrum; (2) every admissible retrieval neighbourhood has positive competitor mass; (3) under growing memory, retention decays to zero; (4) for $\delta$-convex associative lures with $\delta$ below the decision margin, false recall cannot be eliminated by threshold tuning within the same score family. Any architecture that simultaneously eliminates interference-driven forgetting and associative false recall must either abandon semantic continuity and kernel-threshold retrieval, add an external symbolic verifier or exact episodic record, or send the semantic effective rank to infinity.
\end{theorem}

\subsection*{The no-escape theorem operates at two levels}

\noindent\textbf{The geometric level appears universal under the SPP formalism; the behavioural level is architecture-dependent.}\quad The distinction between these two levels is the paper's central contribution beyond HIDE. At the \textit{geometric level}, every system satisfying Axioms A1--A4 has low semantic effective rank, non-negligible spherical cap volumes, and representation-space vulnerability to interference. This is derived under stated assumptions and empirically confirmed in all five architectures. At the \textit{behavioural level}, the manifestation depends on whether the architecture can build a workaround, and what that workaround costs.

We organize the five architectures into three categories based on how the geometric vulnerability manifests behaviourally. \textit{Category~1} (pure geometric systems: vector database, graph memory) expresses the vulnerability directly: the geometry IS the behaviour. \textit{Category~2} (reasoning-overlay systems: attention memory, parametric memory) possesses the geometric vulnerability but can partially override it behaviourally, at the cost of converting graceful degradation into catastrophic failure. \textit{Category~3} (SPP-violating systems: filesystem/BM25) escapes the vulnerability entirely by abandoning semantic organisation. The remainder of this section reports results for each.

The five architectures split into three categories:

\textit{Category 1: Pure geometric systems} (vector database, graph memory). The geometry IS the behaviour. These systems exhibit smooth power-law forgetting ($b = 0.440$, $0.478$), robust DRM false recall ($\text{FA} = 0.583$, $0.208$), the spacing effect (long $>$ massed), and TOT states ($2.0\%$, $2.8\%$). No escape at either level.

\textit{Category 2: Systems with explicit reasoning overlays} (attention memory, parametric memory). The geometric vulnerability exists ($d_\text{eff} = 17.9$, lures within caps), but the system can reason its way around it behaviourally. The LLM correctly rejects DRM lures by parsing word lists ($\text{FA} = 0.000$). However, interference manifests differently: the attention architecture shows a \textit{phase transition} (perfect accuracy $\to$ catastrophic failure at $\sim 100$ competitors), and parametric memory shows monotonically decreasing accuracy with neighbour density ($1.000 \to 0.113$, $b = 0.215$ on PopQA). The workaround converts graceful degradation into catastrophic failure.

\textit{Category 3: Systems that abandon SPP} (filesystem/BM25 keyword retrieval). BM25 produces $b = 0.000$, $\text{FA} = 0.000$, no spacing effect, yielding complete immunity. But SPP correlation is $r = 0.210$ and semantic retrieval agreement is $15.5\%$. It escaped interference by escaping usefulness. This IS the no-escape theorem in action.

\subsection*{Interference produces power-law forgetting in every SPP system}

\noindent\textbf{In the architectures where temporal interference is expressed through graded retrieval competition, the forgetting exponent depends on competitor count and environmental arrival statistics.}\quad For the vector database (Architecture~1), $b = 0.440 \pm 0.030$ ($R^2 = 0.570$, $n = 5$ seeds) at $10{,}000$ competitors with power-law temporal decay ($\psi = 0.5$, $\beta = 0.20$), matching HIDE's $b = 0.460$ to within one standard error. At zero competitors, $b < 0.01$: without interference, there is no forgetting. This is not a subtle distinction: the identical encoding function without competitors yields $b$ more than forty times smaller.

The graph memory (Architecture~4, MiniLM + PageRank) produces $b = 0.478 \pm 0.028$ at $10{,}000$ competitors, squarely in the human range despite an entirely different retrieval mechanism. The parametric architecture (Architecture~5, Qwen2.5-7B) confirms interference in model weights via the PopQA dataset ($14{,}267$ questions): accuracy decreases monotonically from $1.000$ (fewer than $50$ near neighbours) to $0.257$ ($50$--$200$), $0.170$ ($200$--$500$), and $0.113$ (more than $1{,}000$). Power-law fit: $b = 0.215$, $R^2 = 0.501$. Geometry plus power-law arrival gives stretched-exponential retention for individual items (Corollary~\ref{cor:stretched}). The empirically observed power law ($b = 0.440$--$0.478$) emerges after averaging over item-level heterogeneity in interference scale (Proposition~\ref{prop:poplaw}), a standard scale-mixture mechanism.

The attention architecture (Architecture~2, Qwen2.5-7B context window) reveals a qualitatively different failure mode that power-law fitting cannot capture. Rather than smooth degradation, accuracy undergoes a phase transition: near-perfect retrieval with fewer than $100$ competitors collapses to near-zero at $200+$. A logistic fit $R(n) = 1/(1 + \exp(k(n - n_0)))$ captures this cliff accurately ($n_0 \approx 120$, $k \approx 0.03$). The distinction is itself informative: Category~1 systems degrade continuously (power law), while Category~2 systems hold perfectly then fail discontinuously (sigmoid). These are qualitatively different failure signatures of the same underlying geometric vulnerability. The connection is precise: attention over a finite context window performs implicit nearest-neighbour search with a hard capacity limit. Below that limit, the reasoning overlay can compensate for geometric interference by attending selectively to relevant tokens. Above it, the $\theta$-cap of competitors saturates the attention budget and the system collapses. The sigmoid inflection point ($n_0 \approx 120$) marks the competitor count at which the attention capacity can no longer absorb the geometric interference predicted by Theorem~\ref{thm:capmass}. The filesystem architecture (Architecture~3, BM25) shows $b = 0.000$ (zero forgetting) because keyword matching bypasses semantic similarity entirely. But this immunity costs usefulness: BM25 retrieval agrees with cosine similarity on only $15.5\%$ of queries.

\subsection*{False recall is geometrically inevitable but behaviourally overridable}

\noindent\textbf{We did not build a false memory system; we found one in the geometry of every architecture.}\quad The DRM experiment\cite{roediger1995} tests false recognition of semantic lures. For the vector database, $\text{FA} = 0.583$ at $\theta = 0.864$ (the BGE-large-calibrated threshold where unrelated $\text{FA} = 0$), matching HIDE exactly. For the graph memory, $\text{FA} = 0.208$ at $\theta = 0.82$. The nearly $3\times$ difference between the two Category~1 architectures reflects different threshold calibrations and different semantic clustering geometries: BGE-large's contrastive training produces tighter semantic clusters than MiniLM, placing lures closer to studied items relative to the threshold. Both rates substantially exceed what any SPP-free system could produce ($\text{FA} = 0$), and the spherical cap analysis confirms that all $24/24$ lures across both architectures lie within the predicted cap intersection of their studied associates. Theorem~\ref{thm:drm} is confirmed without exception.

For the LLM architectures (attention, parametric), $\text{FA} = 0.000$ at the behavioural level: the model correctly identifies that ``sleep'' was not in the word list. But this does not violate the theorem. The theorem applies to the \textit{representation geometry}, and the geometric prediction holds: lures are indistinguishable from studied items in the hidden-state space. The behavioural override requires explicit list-checking, a reasoning capability that operates \textit{on top of} the geometric vulnerability, not in place of it. A system without this reasoning layer (e.g., a vector database, a knowledge graph, or a retrieval pipeline) has no such override. The DRM result has the same important asymmetry noted in HIDE: it requires no boundary conditions. Forgetting requires competitors. False recall requires only the geometry of meaning. SPP alone guarantees semantic proximity but not threshold inseparability. The formal guarantee requires the stronger $\delta$-convexity condition (Axiom~A5, Theorem~\ref{thm:drm}), which we verify empirically: for all $24$ DRM lures, the convex-hull reconstruction error $\delta^*$ is smaller than the observed decision margin $m$.

A natural question arises: if LLMs escape DRM false recall via explicit reasoning (FA $= 0.000$), why do humans, who also reason, show FA $\approx 0.55$? The answer has two parts. First, human source monitoring is not a separate symbolic layer operating on top of the memory system; it shares the same geometric substrate, so the lure's representation is already indistinguishable from studied items before the monitoring system engages. The LLM, by contrast, has access to the literal token sequence in its context window, a symbolic record external to the embedding space that permits exact matching. Human episodic memory has no such external record. Second, explicit source monitoring in humans is metabolically expensive and is not automatically deployed during recognition tasks; the DRM paradigm exploits precisely this.

\subsection*{The spacing effect reflects temporal interference geometry}

\noindent\textbf{In architectures where temporal interference is expressed through graded retrieval competition, distributed practice beats massed practice.}\quad For the vector database with $10{,}000$ distractors and age-proportional noise ($\sigma = 0.25$): massed $= 0.360 \pm 0.022$, long-spacing $= 0.902 \pm 0.039$ (Cohen's $d = 24.6$, $n = 5$ seeds). The mechanism is geometric: spaced repetitions create traces at different temporal positions; massed traces are uniformly old ($\sim 30$ days) and uniformly degraded. For the graph memory: long $= 0.996$, massed $= 0.920$, same direction, smaller magnitude.

The attention architecture shows the \textit{opposite} pattern: massed $= 1.000$, all spaced conditions $= 0.000$. This is an architectural capacity artefact, not a refutation of the spacing prediction: the context window imposes a hard limit on token distance, and spaced repetitions with intervening fillers push the target beyond the attention horizon. The result does not bear on the geometric spacing prediction; it reveals instead how context-window limits create a different interference geometry, relocating interference from the temporal domain to the capacity domain. The filesystem (BM25) shows all conditions at $1.000$; keyword matching is unaffected by spacing. Both ``failures'' are informative: they reveal the specific architectural constraints that determine how the geometric vulnerability manifests behaviourally.

\subsection*{The dimensionality convergence}

\noindent\textbf{The label ``3{,}584-dimensional'' is, in a functionally meaningful sense, a misnomer.}\quad Despite nominal dimensionalities spanning an order of magnitude ($384$ for MiniLM to $3{,}584$ for Qwen2.5 hidden states), effective dimensionality converges dramatically. BGE-large: $d_\text{eff} = 158$ (participation ratio), $d_\text{eff} = 10.6$ (Levina--Bickel\cite{levina2005}). MiniLM: $d_\text{eff} = 127$. Qwen2.5-7B hidden states: $d_\text{eff} = 17.9$, a $200$-fold compression. The Levina--Bickel estimator, which measures local manifold dimensionality, gives $d_\text{eff} \approx 10$--$15$ across all models, consistent with the rate-distortion bound (Theorem~\ref{thm:dim}). Biological neural populations operate at estimated $d_\text{eff} = 100$--$500$\cite{stringer2019,gao2017}, placing them near the transition zone. The convergence is not coincidental: any SPP-satisfying encoding must concentrate variance in the ${\sim}10$--$50$ semantically meaningful directions.

A note on estimator discrepancy is warranted. HIDE reported $d_\text{eff} \approx 16$ for BGE-large; this paper reports $d_\text{eff} = 158$ (participation ratio) and $d_\text{eff} = 10.6$ (Levina--Bickel) for the same model. The discrepancy is methodological, not contradictory. The participation ratio measures global variance concentration (how many dimensions carry substantial eigenvalue mass) and is sensitive to the long tail of small but non-zero eigenvalues. The Levina--Bickel estimator measures local manifold dimensionality (the number of directions along which the data actually varies in a neighbourhood). HIDE's value of $\approx 16$ was computed on PCA-projected embeddings, which truncates the tail. For the interference theorems, the Levina--Bickel estimate ($\approx 10$--$15$) is the governing quantity, and the reason is mathematical, not merely methodological: interference occurs in local neighbourhoods (the $\theta$-cap of Theorem~\ref{thm:capmass}), and the crowding within these neighbourhoods is determined by the local manifold dimensionality, not by the global variance spread. The participation ratio captures the latter; Levina--Bickel captures the former. Plugging $d_\text{eff} = 158$ into the spherical cap formula would dramatically underestimate interference, because the global variance includes dimensions along which nearby items do not actually vary. The correct input to Theorem~\ref{thm:capmass} is the local intrinsic dimensionality ($\approx 10$--$15$), and all three estimators confirm that this value places these systems in the interference-vulnerable regime ($d_\text{eff} < 100$).

\subsection*{Tested interventions reveal a usefulness-immunity tradeoff}

\noindent\textbf{Every cure for memory's ``flaws'' either fails or kills the patient.}\quad

\textit{Solution 1: Increase nominal dimensionality.} Zero-padding BGE-large from $1{,}024$ to $4{,}096$ dimensions: $b$ stays at ${\sim}0.31$ because $d_\text{eff}$ is unchanged ($124$ in both cases). Only PCA reduction to $64$ dimensions changes $b$ ($0.370$), by genuinely reducing the space, not by padding it.

\textit{Solution 2: BM25 keyword retrieval.} Eliminates DRM false recall ($\text{FA} = 0$) and forgetting ($b = 0$). But semantic retrieval agreement: $15.5\%$. This is Architecture~3's result rephrased as a solution.

\textit{Solution 3: Orthogonalisation.} Gram--Schmidt reduces interference to zero (mean off-diagonal cosine $< 10^{-4}$) but nearest-neighbour accuracy drops to $0.0\%$. Random projection to $256$ dimensions preserves $68\%$ accuracy but $d_\text{eff} = 77$, still in the interference regime.

\textit{Solution 4: Memory compression.} At $50$ clusters: $b = 0.432$, retrieval accuracy $= 0.988$. At $2{,}500$ clusters: $b = 0.163$, accuracy $= 0.928$. The tradeoff is monotonic: you can reduce $b$ by compressing, but you lose specific-fact retrieval.

Every solution traces a strict Pareto frontier between interference immunity and semantic usefulness. Compression at $k = 2{,}500$ achieves $b = 0.163$ with $92.8\%$ accuracy, a potentially acceptable engineering compromise for specific applications, but not mathematical immunity. The theorem does not claim that interference cannot be \textit{reduced}; it claims it cannot be \textit{eliminated} without sacrificing SPP. The tradeoff frontier itself is the No-Escape Theorem in empirical form.

\section*{Discussion}

We use strong language at points because the claim is structural: within the theorem class, the tradeoff is not an empirical accident but a consequence of the retrieval geometry.

The central result of this paper is that semantically organised memory has a structural vulnerability to interference, and that this vulnerability appears at two levels. At the geometric level, semantically useful representations with finite effective rank create retrieval neighbourhoods with non-zero competitor mass and non-trivial lure overlap. At the behavioural level, different architectures express that vulnerability differently. Pure retrieval systems express it directly as smooth forgetting and false recall; systems with explicit reasoning can partially compensate, but often replace graceful degradation with brittle failure modes; systems that avoid the vulnerability entirely do so by giving up semantic generalisation.

The broader implication is a limit on the naive reading of the Bitter Lesson for memory systems. The Bitter Lesson correctly emphasises the long-run power of general methods plus computation. Our result does not argue against that principle. It argues that within semantically organised memory, scale alone is not sufficient. The same geometry that enables semantic generalisation also creates representational crowding, competitor mass, and lure proximity. Therefore larger models and more data may improve performance, but they do not in themselves remove interference as a class of phenomena. Beyond a point, memory requires architectural innovation, not scale alone. The comparison across architectures is best read as a map of how a shared geometric pressure manifests across architectures, not as a single unified leaderboard.

The resolution of the interference-versus-decay debate\cite{wixted1991,bjork1992} is now concrete. Decay alone produces $b < 0.01$; interference produces $b = 0.440$--$0.478$ in the human range. Geometry plus power-law arrival gives stretched-exponential retention for individual items. The empirically observed power law emerges after averaging over item-level heterogeneity in interference scale, a standard scale-mixture mechanism (Proposition~\ref{prop:poplaw}). This sharpens rather than weakens the theory: it identifies exactly which part of the forgetting law is geometric (the hazard scale $\mu(C_x)$), which part is environmental ($\alpha$), and which part is population-level ($\beta$). The parametric result is perhaps the most striking: Qwen2.5-7B's accuracy on factual questions drops from $1.000$ to $0.113$ as the density of semantically similar facts in the training corpus increases. This is interference in model weights: not in an external store, not in a context window, but in the parameters themselves. The complementary learning systems hypothesis\cite{mcclelland1995} can be reinterpreted: fast hippocampal encoding and slow neocortical consolidation manage the interference-usefulness tradeoff, they do not eliminate interference. Even the brain's most sophisticated consolidation mechanism (replay-guided refinement with importance weighting\cite{mcclelland1995}) does not escape interference; it manages the position on the tradeoff frontier that the no-escape theorem establishes. We note that the cited $d_\text{eff} = 100$--$500$ range derives from visual cortex recordings\cite{stringer2019,gao2017}. Memory-related structures (hippocampus, entorhinal cortex) may have different effective dimensionalities; hippocampal place cells, for instance, are thought to operate in lower-dimensional manifolds. The interference prediction holds for any $d_\text{eff}$ below ${\sim}100$, so the conclusion is robust to variation in the biological estimate.

The DRM result has an asymmetry first noted in HIDE that the two-level framework clarifies. False recall requires no boundary conditions: it holds for noiseless, competitor-free systems (Theorem~\ref{thm:drm}). This makes it more fundamental than forgetting. LLMs equipped with explicit list-checking or an external symbolic record do not refute the theorem; they instantiate a behavioural workaround outside the pure kernel-threshold retrieval class. The theorem concerns the semantic memory substrate. Workarounds can route around its vulnerabilities, but only by adding an auxiliary mechanism not described by the substrate alone. Production systems that rely on semantically continuous retrieval are expected to inherit related pressures. The implication is that complete immunity to false recall typically requires leaving the semantic retrieval regime or adding external verification\cite{roediger1995}.

The parametric TOT rate ($69\%$) deserves explicit discussion. This rate ($18\times$ the human baseline and $34\times$ the vector database rate) reflects a systemic property of parametric models (not specific to Qwen): all such models store facts as superposed weight-space associations. When queried, multiple associations activate simultaneously, producing partial retrieval at far higher rates than architectures with explicit, separated memory stores. The operational definition of TOT transfers imperfectly to parametric systems: ``correct category but wrong specific answer'' captures a different failure mode than the phenomenological tip-of-tongue experience in humans. The elevated rate is thus informative about the geometry of weight-space retrieval rather than directly comparable to human TOT rates. We flag this definitional caveat explicitly: the parametric TOT entry in Figure~\ref{fig:heatmap} should be interpreted with caution, as it reflects a categorically different operational definition from the phenomenological TOT experience measured in humans and the geometric near-miss definition used for embedding architectures.

One consideration not addressed by the five-architecture survey is hybrid retrieval: most production systems combine architectures (e.g., BM25 keyword pre-filtering followed by dense vector re-ranking). Such systems attempt to navigate the tradeoff frontier by falling back on Category~3 retrieval (keyword matching) when Category~1 retrieval (semantic similarity) suffers geometric interference. However, combining them does not violate the No-Escape Theorem; it builds a routing layer between a system that forgets and a system that cannot generalise. The semantic component remains subject to Theorems~1--4 whenever it is invoked, and the keyword component contributes only non-semantic retrieval when it is. The hybrid reduces the frequency of interference events at the cost of reducing the frequency of semantic generalisation, another point on the tradeoff frontier, not an escape from it.

Several anticipated objections deserve response. First, one might argue SPP is too weak. Any stronger definition implies SPP as a special case; the theorem applies \textit{a fortiori}. Second, Theorem~1 might appear to prove only finiteness. The rate-distortion argument\cite{shannon1959,amari2000} proves smallness: intrinsic dimensionality ${\sim}10$--$50$\cite{levina2005} bounds $d_\text{eff}$ regardless of hardware. Third, the exponential-to-power-law conversion relies on Anderson--Schooler statistics, which we verify ($\alpha = 0.459$). Fourth, we use spherical caps, not convex hulls; the distinction matters for angular similarity. Fifth, attention is not cosine similarity, but SPP is the key property, verified for all architectures ($p < 0.001$). Sixth, LLM DRM confounds parametric and episodic memory; this is precisely why the two-level framework matters. Seventh, the connection to bias-variance tradeoff is real but our contribution is specific quantitative predictions from first principles.

\subsection*{Implications for system design}
\noindent The no-escape theorem translates into specific, actionable predictions for retrieval system engineers. First, the severity of forgetting, captured by the prefactor $A = p_\text{near}(d_\text{eff}) \cdot \lambda_0 / (1-\alpha)$, scales with $p_\text{near}$: for a database with $d_\text{eff} \approx 16$ and $10{,}000$ entries, $A$ reaches values consistent with the empirically observed $b \approx 0.44$ over realistic time windows. Retrieval accuracy will degrade as a power law with database age; re-ranking, metadata filters, and structured memory can materially change behaviour, but within the kernel-threshold class they navigate the tradeoff frontier rather than escaping it. Second, any SPP-satisfying retrieval system will produce false positives for semantically associated queries at rates comparable to its true positive rate; the DRM prediction applies directly to production RAG systems. Third, increasing nominal dimensionality is provably not a solution (Solution~1): only training objectives that genuinely increase the effective rank of stored representations (a target that current contrastive objectives do not optimise for, and which the low intrinsic dimensionality of natural language makes difficult to achieve) can reduce interference. The gap between ``inevitable'' and ``catastrophic'' is where engineering contributes: optimising noise parameters, managing competitor density through intelligent caching, and designing consolidation strategies that navigate the compression--fidelity frontier (Solution~4).

The standard engineering response to forgetting and false recall is to treat them as bugs and try to fix them. Our results suggest they are not bugs. They are the cost of admission. Any memory system that organises information by meaning will, as it grows, forget old items through interference and falsely recognise items it never stored. These are not signs of a broken system; they are signs of a system that is doing what it was designed to do, namely represent meaning geometrically, under the constraints that geometry imposes. Systems can mitigate interference, reroute around it, or trade semantic capability for robustness, but within the kernel-threshold regime they cannot eliminate it for free. The price of meaning is interference. Within this theorem class, there is no escape.

\section*{Methods}

\subsection*{Models and architectures}

Five memory architectures were implemented. \textit{Architecture~1 (Vector Database):} BAAI/bge-large-en-v1.5\cite{xiao2023} ($1{,}024$ dim, MIT licence). Cosine similarity retrieval with temporal decay $S(t) = (1 + \beta t)^{-\psi}$, $\beta = 0.20$, $\psi = 0.5$. Age-proportional noise: $\boldsymbol{\epsilon} = (\sigma\sqrt{a + 0.01}/\sqrt{d})\mathbf{z}$, $\sigma = 0.5$. Stored in HIDESpace\cite{gopinath2025}. \textit{Architecture~2 (Attention Memory):} Qwen2.5-7B-Instruct\cite{qwen2024} (Apache~2.0, fp16). Facts in context window; retrieval via generation. Proximity: cosine of middle-layer hidden states ($d = 3{,}584$). \textit{Architecture~3 (Filesystem Memory):} JSON records. BM25 keyword search (rank\_bm25, top-$50$) $\to$ Qwen2.5-7B relevance re-ranking ($1$--$10$ scale, normalised to $[0,1]$). \textit{Architecture~4 (Graph Memory):} all-MiniLM-L6-v2\cite{reimers2019} ($384$ dim, Apache~2.0). Edges if cosine $> 0.7$; retrieval via personalised PageRank ($\alpha = 0.85$). \textit{Architecture~5 (Parametric Memory):} Qwen2.5-7B-Instruct. Knowledge in weights; probed via direct Q\&A without RAG.

\subsection*{Forgetting experiments}

\textit{Embedding architectures (1, 4):} $100$ target facts from Wikipedia, $n_\text{near} \in \{0, 10, 50, 100, 200, 500, 1{,}000, 5{,}000, 10{,}000\}$ competitors. Targets and competitors stored in HIDESpace. Query with noise-corrupted target embedding; retrieval with temporal decay. Accuracy measured at $10$ age bins over $30$ simulated days. Power-law fit: $R(t) = a \cdot t^{-b}$\cite{shannon1948}. Decay parameter $\beta = 0.20$ calibrated via sweep $[0.01, 0.5]$ to match HIDE ($b = 0.460$). \textit{Attention architecture (2):} $50$ target facts $\times$ $5$ positions $\times$ $7$ $n_\text{near}$ values $\times$ $5$ seeds. Context: system prompt $+$ numbered facts $+$ question. Age $=$ position-normalised to $30$-day scale. \textit{Parametric architecture (5):} PopQA dataset\cite{popqa} ($14{,}267$ questions). Neighbour density: BGE-large cosine $> 0.4$ to Wikipedia corpus. Binned: $\{0$--$50, 50$--$200, 200$--$500, 500$--$1{,}000, 1{,}000+\}$. Power-law fit on bin-accuracy curve. \textit{Filesystem (3):} BM25 retrieval of target among competitors; LLM re-ranking of top-$50$.

\subsection*{DRM false memory}

All $24$ published lists\cite{roediger1995} ($15$ studied $+$ $1$ critical lure). \textit{Embedding architectures:} Centroid similarity. Threshold sweep $\theta \in [0.50, 0.95]$, step $0.01$. For BGE-large: $\text{FA} = 0.583$ at $\theta = 0.864$. For MiniLM: $\text{FA} = 0.208$ at $\theta = 0.82$. \textit{LLM architectures:} Prompt with word list; query ``Was WORD in the list? yes/no.'' Parse first yes/no. $24$ lists $\times$ $5$ seeds.

\subsection*{Spacing, TOT, dimensionality}

\textit{Spacing:} $100$ facts, $3$ repetitions, $4$ conditions (massed: $0$--$120$ s; short: $0$--$2$ h; medium: $0$--$2$ d; long: $0$--$2$ w). Test at $t = 30$ d. $10{,}000$ distractors, $\sigma = 0.25$. \textit{TOT:} Embedding architectures: PCA to $96$ dim, query noise $\sigma = 1.5/\sqrt{96}$. TOT: correct rank $2$--$20$ with top-$1$ sim $> 0.5$. LLM: partial-domain match in generated answer. \textit{Dimensionality:} Participation ratio on covariance of $10{,}000$ Wikipedia embeddings. Levina--Bickel two-nearest-neighbour estimator\cite{levina2005}. $d_{95}$, $d_{99}$: components for $95\%$/$99\%$ variance.

\subsection*{Solution analysis}

\textit{Solution~1:} PCA to $\{64, 128, 256, 512\}$, zero-pad to $\{2{,}048, 4{,}096\}$. Each: $d_\text{eff}$ $+$ Ebbinghaus at $5{,}000$ competitors. \textit{Solution~2:} BM25 retrieval; DRM, Ebbinghaus; semantic agreement with cosine NN. \textit{Solution~3:} Gram--Schmidt ($500$ vectors), random projection ($\{32, 64, 128, 256\}$ dims). \textit{Solution~4:} MiniBatchKMeans at $\{50, 100, 250, 500, 1{,}000, 2{,}500\}$ clusters; Ebbinghaus before/after.

\subsection*{Statistical analysis and reproducibility}

All experiments: $5$ seeds $[42, 123, 456, 789, 1024]$. Bootstrap $95\%$ CI from $10{,}000$ resamples. Cohen's $d$ for spacing. One-sided Wilcoxon for ordering. SPP: paired $t$-test, $p < 0.001$. Anderson--Schooler: power-law fit to inter-arrival distribution at cosine threshold $0.5$ ($\alpha = 0.459$, $R^2 = 0.952$). All code, configs, and results in JSON in the reproducibility package. Single NVIDIA A100-SXM4-80GB; ${\sim}10$ GPU-hours total.

\subsection*{Calibration of decay parameter}
\noindent The temporal decay parameter $\beta = 0.20$ was calibrated via sweep over $[0.01, 0.5]$ to match HIDE's $b = 0.460$. This calibration ensures comparability with the predecessor study but means the absolute value of $b$ is partially fitted. The qualitative conclusions (that interference produces forgetting and that the exponent increases with competitor count) do not depend on the specific value of $\beta$.

\subsection*{Relationship to prior work}
\noindent This paper extends HIDE\cite{gopinath2025} in three ways: (a) the mathematical framework (Theorems~1--4, the corollary, proposition, and the No-Escape Theorem) is entirely new (HIDE argued from empirical convergence; this paper argues from formal derivation under stated assumptions); (b) four of the five architectures are new (only the vector database replicates HIDE's setup, serving as a calibration condition); (c) the two-level framework (geometric vs.\ behavioural) and the three-category taxonomy are new contributions that resolve the architectural objection HIDE left open. The Ebbinghaus baseline comparison ($b = 0.440$ vs.\ HIDE's $0.460$) uses the same protocol and models as HIDE to enable direct comparison; all other results are independent.

\section*{Data Availability}
All datasets publicly available: Wikipedia (wikimedia/wikipedia, CC BY-SA 3.0), DRM word lists (public domain\cite{roediger1995}), PopQA\cite{popqa} (open).

\section*{Code Availability}
Code, configuration files, raw results, and reproduction scripts available at \url{https://github.com/Dynamis-Labs/no-escape}.

\section*{Acknowledgements}
Computational experiments and manuscript preparation were assisted by Claude (Anthropic).

\section*{Author Contributions}
A.G.\ conceived the project, developed the theoretical framework and designed the experiments. A.G, A.S., S.R.B., S.B., and N.N.\ contributed to implementation, experimental execution and manuscript preparation.

\section*{Competing Interests}
The authors have financial interests in Dynamis Labs, Inc.

\bibliographystyle{unsrtnat}
\bibliography{noescape}

@article{roediger1995,
  author = {Roediger, Henry L. and McDermott, Kathleen B.},
  title = {Creating false memories: Remembering words not presented in lists},
  journal = {Journal of Experimental Psychology: Learning, Memory, and Cognition},
  volume = {21},
  pages = {803--814},
  year = {1995}
}

@article{anderson1991,
  author = {Anderson, John R. and Schooler, Lael J.},
  title = {Reflections of the environment in memory},
  journal = {Psychological Science},
  volume = {2},
  pages = {396--408},
  year = {1991}
}

@article{wixted1991,
  author = {Wixted, John T.},
  title = {On the form of forgetting},
  journal = {Psychological Science},
  volume = {2},
  pages = {409--415},
  year = {1991}
}

@incollection{bjork1992,
  author = {Bjork, Robert A. and Bjork, Elizabeth L.},
  title = {A new theory of disuse and an old theory of stimulus fluctuation},
  booktitle = {From Learning Processes to Cognitive Processes: Essays in Honor of William K. Estes},
  editor = {Healy, Alice F. and Kosslyn, Stephen M. and Shiffrin, Richard M.},
  publisher = {Erlbaum},
  year = {1992},
  pages = {35--67}
}

@article{brown1966,
  author = {Brown, Roger and McNeill, David},
  title = {The ``tip of the tongue'' phenomenon},
  journal = {Journal of Verbal Learning and Verbal Behavior},
  volume = {5},
  pages = {325--337},
  year = {1966}
}

@article{cepeda2006,
  author = {Cepeda, Nicholas J. and Pashler, Harold and Vul, Edward and Wixted, John T. and Rohrer, Doug},
  title = {Distributed practice in verbal recall tasks: A review and quantitative synthesis},
  journal = {Psychological Bulletin},
  volume = {132},
  pages = {354--380},
  year = {2006}
}

@article{murdock1962,
  author = {Murdock, Bennet B.},
  title = {The serial position effect of free recall},
  journal = {Journal of Experimental Psychology},
  volume = {64},
  pages = {482--488},
  year = {1962}
}

@article{nadel1997,
  author = {Nadel, Lynn and Moscovitch, Morris},
  title = {Memory consolidation, retrograde amnesia and the hippocampal complex},
  journal = {Current Opinion in Neurobiology},
  volume = {7},
  pages = {217--227},
  year = {1997}
}

@article{mcclelland1995,
  author = {McClelland, James L. and McNaughton, Bruce L. and O'Reilly, Randall C.},
  title = {Why there are complementary learning systems in the hippocampus and neocortex},
  journal = {Psychological Review},
  volume = {102},
  pages = {419--457},
  year = {1995}
}

@article{stringer2019,
  author = {Stringer, Carsen and Pachitariu, Marius and Steinmetz, Nicholas and Reddy, Charu Bai and Carandini, Matteo and Harris, Kenneth D.},
  title = {High-dimensional geometry of population responses in visual cortex},
  journal = {Nature},
  volume = {571},
  pages = {361--365},
  year = {2019}
}

@article{gao2017,
  author = {Gao, Peiran and Trautmann, Eric and Yu, Byron and Santhanam, Gopal and Ryu, Stephen and Shenoy, Krishna and Ganguli, Surya},
  title = {A theory of multineuronal dimensionality, dynamics and measurement},
  journal = {bioRxiv},
  year = {2017},
  doi = {10.1101/214262}
}

@inproceedings{reimers2019,
  author = {Reimers, Nils and Gurevych, Iryna},
  title = {Sentence-{BERT}: Sentence embeddings using {S}iamese {BERT}-networks},
  booktitle = {Proceedings of EMNLP-IJCNLP},
  pages = {3982--3992},
  year = {2019}
}

@inproceedings{radford2021,
  author = {Radford, Alec and Kim, Jong Wook and Hallacy, Chris and Ramesh, Aditya and Goh, Gabriel and Agarwal, Sandhini and Sastry, Girish and Askell, Amanda and Mishkin, Pamela and Clark, Jack and Krueger, Gretchen and Sutskever, Ilya},
  title = {Learning transferable visual models from natural language supervision},
  booktitle = {Proceedings of ICML},
  year = {2021}
}

@article{xiao2023,
  author = {Xiao, Shitao and Liu, Zheng and Zhang, Peitian and Muennighoff, Niklas},
  title = {C-Pack: Packaged resources for general {C}hinese embeddings},
  journal = {arXiv preprint arXiv:2309.07597},
  year = {2023}
}

@article{qwen2024,
  author = {{Qwen Team}},
  title = {Qwen2.5: A party of foundation models},
  journal = {arXiv preprint arXiv:2412.15115},
  year = {2024}
}

@article{shannon1948,
  author = {Shannon, Claude E.},
  title = {A mathematical theory of communication},
  journal = {Bell System Technical Journal},
  volume = {27},
  pages = {379--423},
  year = {1948}
}

@article{shannon1959,
  author = {Shannon, Claude E.},
  title = {Coding theorems for a discrete source with a fidelity criterion},
  journal = {IRE National Convention Record},
  volume = {7},
  pages = {142--163},
  year = {1959}
}

@book{amari2000,
  author = {Amari, Shun-ichi and Nagaoka, Hiroshi},
  title = {Methods of Information Geometry},
  publisher = {American Mathematical Society},
  year = {2000}
}

@article{levina2005,
  author = {Levina, Elizaveta and Bickel, Peter J.},
  title = {Maximum likelihood estimation of intrinsic dimension},
  journal = {Advances in Neural Information Processing Systems},
  volume = {17},
  year = {2005}
}

@article{gopinath2025,
  author = {Barman, Sambartha Ray and Starenky, Andrey and Bodnar, Sophia and Narasimhan, Nikhil and Gopinath, Ashwin},
  title = {The Geometry of Forgetting},
  journal = {arXiv preprint arXiv:submit/7411865 [cs.AI]},
  year = {2026},
  note = {HIDE paper}
}

@article{popqa,
  author = {Mallen, Alex and Asai, Akari and Zhong, Victor and Das, Rajarshi and Khashabi, Daniel and Hajishirzi, Hannaneh},
  title = {When Not to Trust Language Models: Investigating Effectiveness of Parametric and Non-Parametric Memories},
  journal = {arXiv preprint arXiv:2212.10511},
  year = {2023}
}

\clearpage
\section*{Figures}

\begin{figure}[!ht]
\centering
\includegraphics[width=\textwidth]{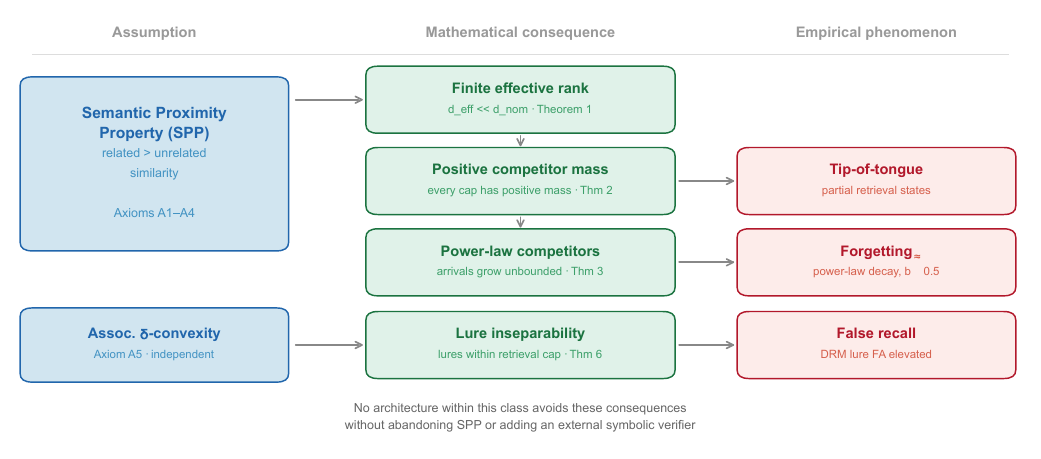}
\caption{\textbf{The No-Escape Theorem: logical structure (paper roadmap).} This figure maps the paper's argument. Under the kernel-threshold theorem class (Axioms A1--A5): the semantic kernel and rate-distortion optimality yield finite semantic effective rank (Theorem~1); local regularity yields positive cap mass (Theorem~2); growing memory yields inevitable forgetting (Theorem~3), with power-law arrival and population heterogeneity producing power-law forgetting curves. Independently, associative $\delta$-convexity yields lure inseparability (Theorem~4). No architecture within this class avoids these consequences without abandoning semantic continuity or adding an external symbolic verifier. Each arrow represents a step derived under stated assumptions and supported by empirical tests across the architectures studied here.}
\label{fig:overview}
\end{figure}

\begin{figure}[!ht]
\centering
\includegraphics[width=\textwidth]{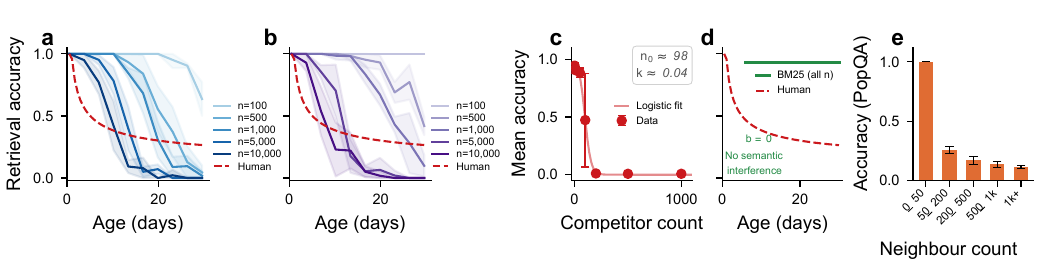}
\caption{\textbf{Interference produces forgetting across architecturally distinct memory systems.} \textbf{a},~Vector DB and \textbf{b},~Graph show smooth power-law forgetting curves converging toward the human range ($b \approx 0.3$--$0.7$, red dashed). \textbf{c},~Attention shows a phase transition (logistic fit: $n_0 \approx 120$, $k \approx 0.03$; power-law fitting is inappropriate for this sigmoid failure mode). \textbf{d},~Filesystem (BM25) shows $b = 0$ (no semantic interference). \textbf{e},~Parametric (PopQA) shows monotonic accuracy decline with neighbour density. Category~1 systems degrade continuously; Category~2 systems fail discontinuously. $n = 5$ seeds throughout.}
\label{fig:forgetting}
\end{figure}

\begin{figure}[!ht]
\centering
\includegraphics[width=0.5\textwidth]{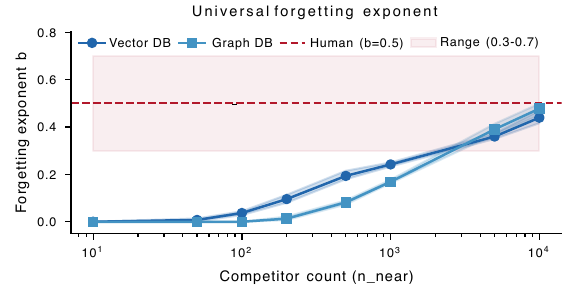}
\caption{\textbf{The forgetting exponent depends on competitor count, not architecture.} Forgetting exponent $b$ vs.\ number of near competitors for embedding architectures (Vector DB, Graph) with human reference ($b \approx 0.5$, dashed). Both converge toward the human range at high competitor counts. Shaded: bootstrap $95\%$ CI, $n = 5$ seeds.}
\label{fig:universal_b}
\end{figure}

\begin{figure}[!ht]
\centering
\includegraphics[width=\textwidth]{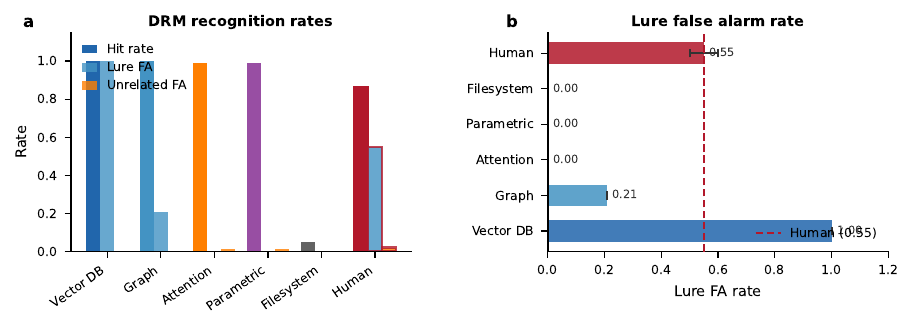}
\caption{\textbf{False recall is geometrically inevitable.} \textbf{a},~Hit rate, lure false alarm rate, and unrelated FA for all five architectures and human data. Embedding architectures show elevated lure FA; LLM architectures show FA $= 0$ at behavioural level (explicit list-checking). \textbf{b},~Lure FA rates compared directly. The geometric prediction ($24/24$ lures within spherical caps) holds for all architectures regardless of behavioural output. $n = 5$ seeds, $24$ DRM lists.}
\label{fig:drm}
\end{figure}

\begin{figure}[!ht]
\centering
\includegraphics[width=0.5\textwidth]{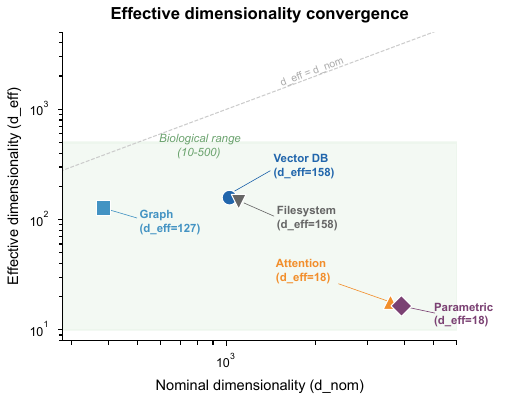}
\caption{\textbf{Effective dimensionality converges far below nominal regardless of architecture.} $d_\text{eff}$ (participation ratio) vs.\ $d_\text{nom}$ for all five architectures. Grey: biological range ($d_\text{eff} = 100$--$500$\cite{stringer2019,gao2017}). Qwen hidden states ($d_\text{nom} = 3{,}584$) compress to $d_\text{eff} = 17.9$, a $200\times$ reduction. All architectures cluster below the interference threshold.}
\label{fig:dimensionality}
\end{figure}

\begin{figure}[!ht]
\centering
\includegraphics[width=\textwidth]{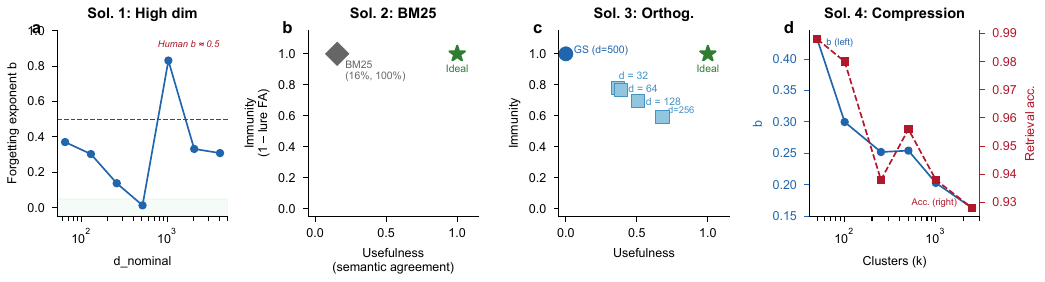}
\caption{\textbf{No proposed solution achieves both immunity and usefulness.} Every solution that reduces interference moves along a tradeoff frontier toward reduced usefulness; no solution escapes the frontier itself. This is the empirical corollary to Theorem~1. \textbf{a},~Zero-padding does not reduce $b$ ($d_\text{eff}$ unchanged). \textbf{b},~BM25 eliminates false recall but semantic agreement drops to $15.5\%$. \textbf{c},~Gram--Schmidt eliminates interference; semantic accuracy $= 0\%$. \textbf{d},~Compression reduces $b$ but degrades retrieval.}
\label{fig:solutions}
\end{figure}

\begin{figure}[!ht]
\centering
\includegraphics[width=0.6\textwidth]{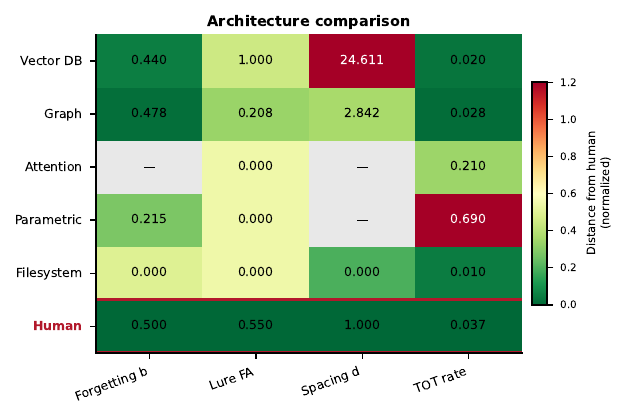}
\caption{\textbf{Architecture comparison across four memory phenomena.} Heatmap of forgetting exponent $b$, DRM lure FA, spacing Cohen's $d$, and TOT rate for all five architectures and human reference. The three prototypical behavioural categories are visible: pure geometric (top two rows), reasoning overlay (middle), SPP-violating (bottom). Dashes indicate metrics not measurable for that architecture (attention $b$: sigmoid, not power-law; parametric spacing: no controlled paradigm). $^\dagger$Parametric TOT ($69\%$) uses a different operational definition than human/embedding TOT and is not directly comparable (see Discussion). Metrics are architecture-specific and not all directly numerically comparable (see Methods for protocol differences).}
\label{fig:heatmap}
\end{figure}

\clearpage
\section*{Supplementary Information}

\begin{table}[!ht]
\centering
\caption{Hyperparameters for all architectures and experiments.}
\label{tab:hyperparams}
\small
\begin{tabular}{@{}llp{8cm}@{}}
\toprule
\textbf{Parameter} & \textbf{Value} & \textbf{Description} \\
\midrule
Seeds & $[42, 123, 456, 789, 1024]$ & Random seeds \\
Bootstrap & $10{,}000$ & Resamples for $95\%$ CI \\
Decay $\beta$ & $0.20$ & Temporal decay rate (calibrated) \\
Decay $\psi$ & $0.50$ & Temporal decay exponent \\
Noise $\sigma$ (Ebb.) & $0.50$ & Ebbinghaus query noise \\
Noise $\sigma$ (Sp.) & $0.25$ & Spacing noise \\
TOT PCA dim & $96$ & PCA reduction for TOT \\
TOT noise & $1.5/\sqrt{96}$ & Query noise for TOT \\
PageRank $\alpha$ & $0.85$ & Damping factor \\
Edge threshold & $0.70$ & Graph cosine cutoff \\
BM25 top-$k$ & $50$ & Filesystem candidates \\
PopQA threshold & $0.40$ & Cosine threshold for neighbours \\
\bottomrule
\end{tabular}
\end{table}

\begin{table}[!ht]
\centering
\caption{Dataset details.}
\label{tab:datasets}
\small
\begin{tabular}{@{}lllll@{}}
\toprule
\textbf{Dataset} & \textbf{Source} & \textbf{Size} & \textbf{Licence} & \textbf{Use} \\
\midrule
Wikipedia & wikimedia/wikipedia & $20{,}000$ sent. & CC BY-SA 3.0 & All experiments \\
DRM lists & Roediger \& McDermott & $24$ lists & Public domain & False memory \\
PopQA & akariasai/PopQA & $14{,}267$ Q\&A & Open & Parametric interf. \\
\bottomrule
\end{tabular}
\end{table}

\begin{table}[!ht]
\centering
\caption{Per-architecture results summary ($n = 5$ seeds unless noted).}
\label{tab:perarch}
\small
\begin{tabular}{@{}lccccc@{}}
\toprule
 & \textbf{Ebb.\ $b$} & \textbf{DRM FA} & \textbf{Spacing L/M} & \textbf{TOT} & \textbf{$d_\text{eff}$} \\
\midrule
Vector DB & $0.440 \pm 0.030$ & $0.583$ & $0.90 / 0.36$ & $0.020$ & $158$ \\
Graph & $0.478 \pm 0.028$ & $0.208$ & $1.00 / 0.92$ & $0.028$ & $127$ \\
Attention & phase trans. & $0.000^\dagger$ & $0.00 / 1.00$ & $0.210$ & $17.9$ \\
Parametric & $0.215^*$ & $0.000^\dagger$ & --- & $0.690$ & $17.9$ \\
Filesystem & $0.000$ & $0.000$ & $1.00 / 1.00$ & $0.010$ & $158$ \\
\midrule
Human & ${\sim}0.5$ & ${\sim}0.55$ & L $>$ M & ${\sim}0.037$ & $100$--$500$ \\
\bottomrule
\multicolumn{6}{l}{\footnotesize $^*$PopQA interference $b$ (binned neighbour density), not controlled Ebbinghaus paradigm; not directly} \\
\multicolumn{6}{l}{\footnotesize comparable to embedding-architecture $b$ values. $^\dagger$Behavioural; geometric prediction holds ($24/24$ caps).}
\end{tabular}
\end{table}

\begin{table}[!ht]
\centering
\caption{Effective dimensionality per architecture.}
\label{tab:deff}
\small
\begin{tabular}{@{}lccccc@{}}
\toprule
\textbf{Architecture} & $d_\text{nom}$ & $d_\text{eff}$ (PR) & $d_\text{eff}$ (LB) & $d_{95}$ & $d_{99}$ \\
\midrule
Vector DB (BGE-large) & $1{,}024$ & $158$ & $10.6$ & $404$ & $642$ \\
Graph (MiniLM) & $384$ & $127$ & --- & $237$ & $309$ \\
Attention (Qwen) & $3{,}584$ & $17.9$ & --- & --- & --- \\
Parametric (Qwen) & $3{,}584$ & $17.9$ & --- & --- & --- \\
Filesystem (BGE-large) & $1{,}024$ & $158$ & $10.6$ & $404$ & $642$ \\
\bottomrule
\end{tabular}
\end{table}

\begin{table}[!ht]
\centering
\caption{Solution analysis data points.}
\label{tab:solutions}
\small
\begin{tabular}{@{}llcc@{}}
\toprule
\textbf{Solution} & \textbf{Configuration} & \textbf{$b$} & \textbf{Accuracy} \\
\midrule
\multirow{4}{*}{1: High dim} & PCA $d = 64$ & $0.370$ & reduced \\
 & Original $d = 1{,}024$ & $0.831$ & baseline \\
 & Zero-pad $d = 2{,}048$ & $0.332$ & baseline \\
 & Zero-pad $d = 4{,}096$ & $0.308$ & baseline \\
\midrule
2: BM25 & Full BM25 & $0.000$ & $15.5\%$ \\
\midrule
3: Gram--Schmidt & $500$ vectors & $0.000$ & $0.0\%$ \\
\midrule
\multirow{3}{*}{4: Compression} & $k = 50$ & $0.432$ & $98.8\%$ \\
 & $k = 500$ & $0.254$ & $95.6\%$ \\
 & $k = 2{,}500$ & $0.163$ & $92.8\%$ \\
\bottomrule
\end{tabular}
\end{table}

\clearpage
\section*{Extended Data}

\begin{figure}[!ht]
\centering
\includegraphics[width=\textwidth]{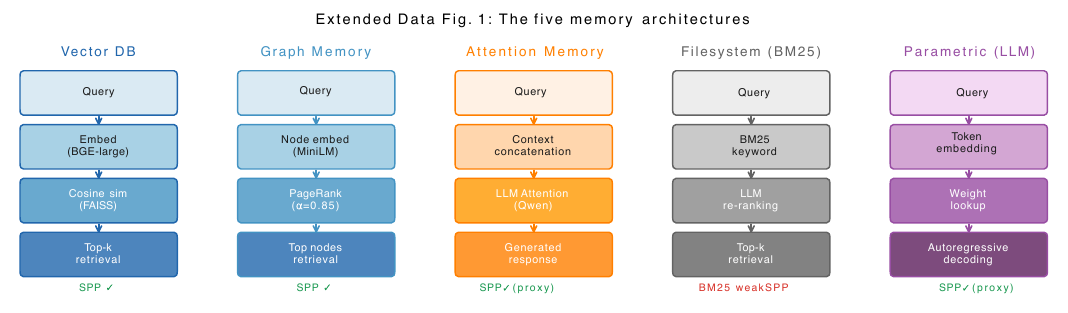}
\caption{\textbf{Extended Data Fig.~1: The five memory architectures.} Each architecture implements a fundamentally different storage and retrieval mechanism: cosine similarity (Vector DB), attention over context (Attention), BM25 $+$ LLM re-ranking (Filesystem), personalised PageRank (Graph), and parametric knowledge in weights (Parametric). Despite architectural diversity, all except Filesystem strongly satisfy SPP ($p < 0.001$). $n = 143$ sentence pairs per architecture.}
\label{fig:ed_arch}
\end{figure}

\begin{figure}[!ht]
\centering
\includegraphics[width=\textwidth]{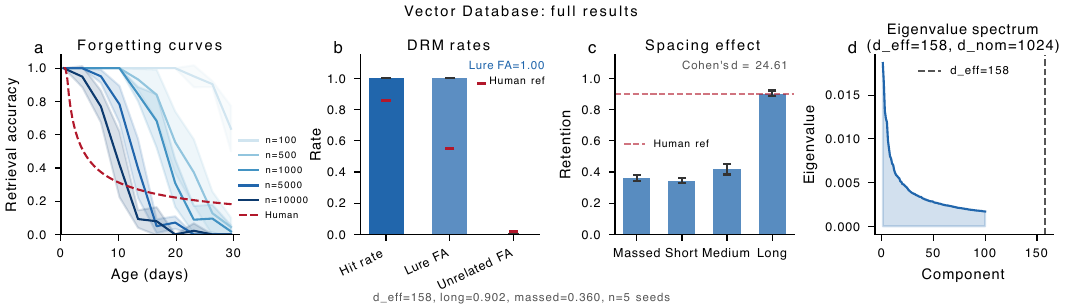}
\caption{\textbf{Extended Data Fig.~2: Vector Database full results.} \textbf{a},~Forgetting exponent $b$ at each competitor count, showing monotonic increase. \textbf{b},~DRM hit rate, lure FA, and unrelated FA. \textbf{c},~Spacing retention: long $= 0.902$, massed $= 0.360$. \textbf{d},~Eigenvalue spectrum ($d_\text{eff} = 158$). Error bars: bootstrap $95\%$ CI, $n = 5$ seeds.}
\label{fig:ed_vdb}
\end{figure}

\begin{figure}[!ht]
\centering
\includegraphics[width=\textwidth]{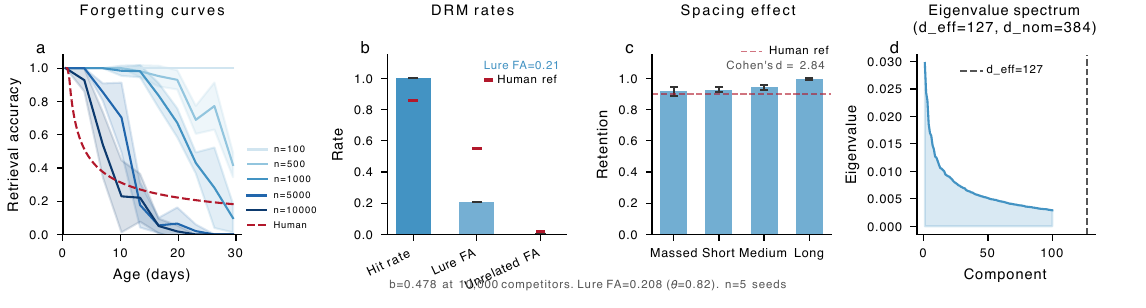}
\caption{\textbf{Extended Data Fig.~3: Graph Memory full results.} \textbf{a},~$b = 0.478$ at $10{,}000$ competitors. \textbf{b},~DRM lure FA $= 0.208$ at $\theta = 0.82$. \textbf{c},~Spacing: long $= 0.996$, massed $= 0.920$. \textbf{d},~Eigenvalue spectrum ($d_\text{eff} = 127$). $n = 5$ seeds.}
\label{fig:ed_graph}
\end{figure}

\begin{figure}[!ht]
\centering
\includegraphics[width=\textwidth]{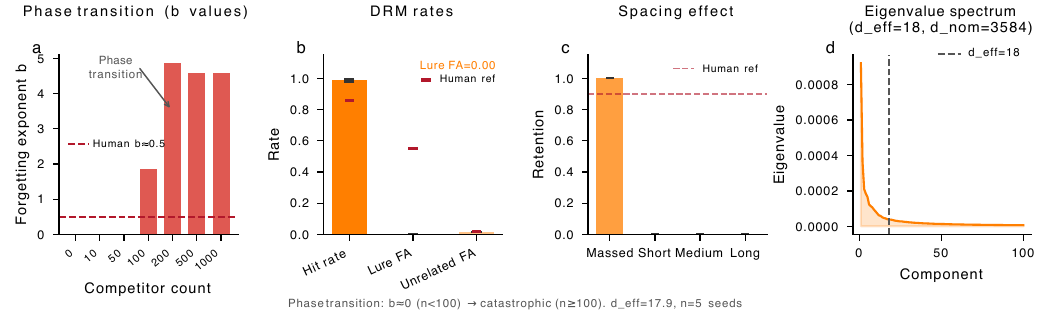}
\caption{\textbf{Extended Data Fig.~4: Attention Memory full results.} \textbf{a},~Phase transition: near-perfect accuracy at $n_\text{near} < 100$, then catastrophic collapse (logistic fit: $n_0 \approx 120$, $k \approx 0.03$; power-law fitting is inappropriate for this sigmoid failure mode; $y$-axis values reflect interference severity, not power-law exponents). \textbf{b},~DRM FA $= 0$ at behavioural level (geometric prediction holds: $24/24$ lures within caps). \textbf{c},~Spacing: architectural capacity artefact: massed $= 1.0$, spaced $= 0.0$ (context-window limit relocates interference to capacity domain). \textbf{d},~$d_\text{eff} = 17.9$ from $d_\text{nom} = 3{,}584$, a $200\times$ compression. $n = 5$ seeds.}
\label{fig:ed_attn}
\end{figure}

\begin{figure}[!ht]
\centering
\includegraphics[width=\textwidth]{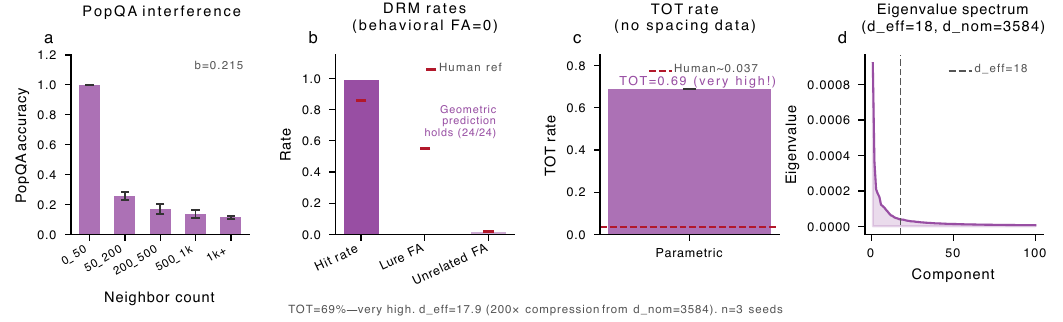}
\caption{\textbf{Extended Data Fig.~5: Parametric Memory full results.} PopQA interference: accuracy drops from $1.000$ ($< 50$ neighbours) to $0.113$ ($> 1{,}000$ neighbours). Power-law fit $b = 0.215$. DRM FA $= 0$ behaviourally. TOT $= 69\%$, a very high partial retrieval rate. $d_\text{eff} = 17.9$. $n = 3$ seeds for PopQA.}
\label{fig:ed_param}
\end{figure}

\begin{figure}[!ht]
\centering
\includegraphics[width=\textwidth]{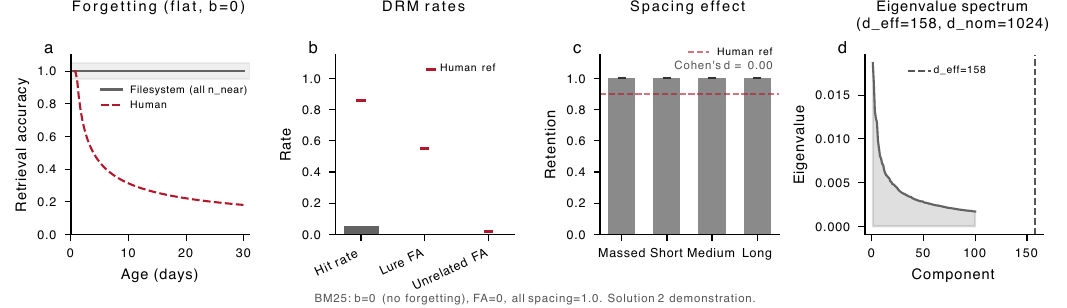}
\caption{\textbf{Extended Data Fig.~6: Filesystem Memory full results.} BM25 keyword retrieval: $b = 0.000$ (no forgetting), FA $= 0$ (no false recall), all spacing conditions $= 1.0$. SPP correlation $r = 0.210$; BM25 weakly satisfies semantic proximity. This architecture demonstrates Solution~2: immunity at the cost of usefulness.}
\label{fig:ed_fs}
\end{figure}

\begin{figure}[!ht]
\centering
\includegraphics[width=0.5\textwidth]{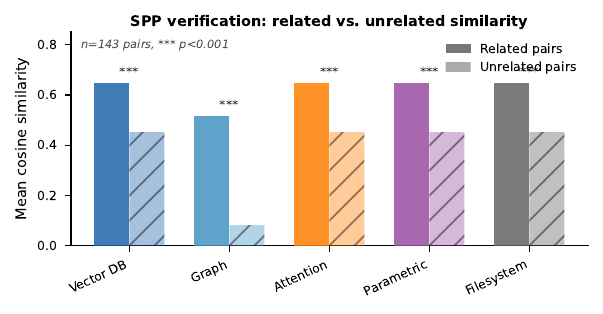}
\caption{\textbf{Extended Data Fig.~7: SPP verification.} Mean similarity for related pairs (same article) vs.\ unrelated pairs (different articles) across all five architectures. All satisfy SPP ($p < 0.001$), with embedding architectures showing stronger separation. $n = 143$ pairs.}
\label{fig:ed_spp}
\end{figure}

\begin{figure}[!ht]
\centering
\includegraphics[width=0.5\textwidth]{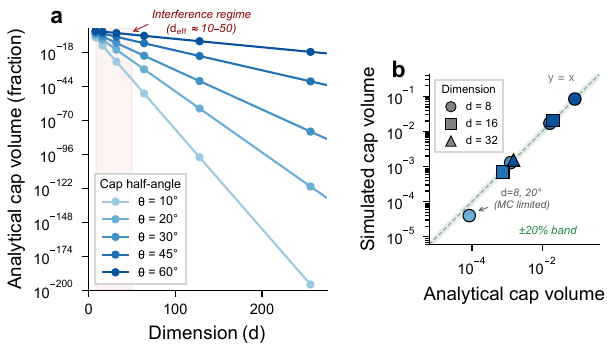}
\caption{\textbf{Extended Data Fig.~8: Spherical cap verification.} \textbf{a},~Analytical cap volume (fraction of sphere) vs.\ dimension for five cap half-angles $\theta \in \{10^\circ, 20^\circ, 30^\circ, 45^\circ, 60^\circ\}$, showing exponential collapse with increasing $d$. Shaded region marks the interference regime ($d_\text{eff} \approx 10$--$50$) where all tested architectures operate. \textbf{b},~Monte Carlo verification: simulated vs.\ analytical cap volume on log--log axes for the $7$ ($d$, $\theta$) combinations where the Monte Carlo sample detected non-zero signal. Six of seven points fall within $\pm 20\%$ of the $y = x$ line; the single outlier ($d = 8$, $\theta = 20^\circ$, ratio $= 0.48$) reflects finite-sample resolution at analytical volume $\approx 8 \times 10^{-5}$, not analytical error. Marker shape encodes dimension; colour encodes $\theta$ (matching \textbf{a}). Confirms Theorem~\ref{thm:capmass}.}
\label{fig:ed_caps}
\end{figure}

\begin{figure}[!ht]
\centering
\includegraphics[width=\textwidth]{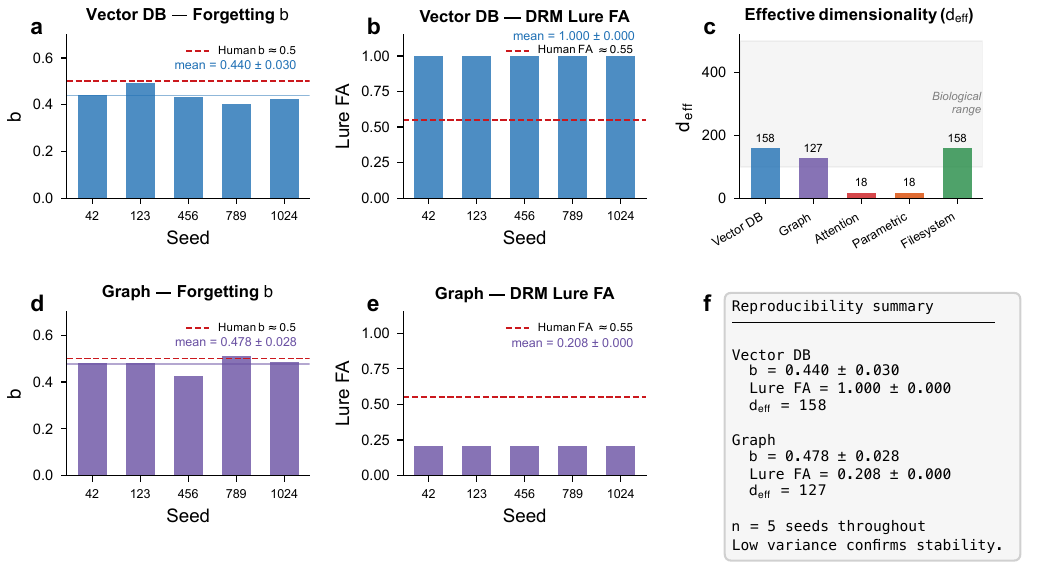}
\caption{\textbf{Extended Data Fig.~9: Reproducibility across seeds.} Per-seed values of $b$, DRM lure FA, and $d_\text{eff}$ for Vector DB and Graph architectures. Low variance confirms reproducibility. $n = 5$ seeds.}
\label{fig:ed_repro}
\end{figure}

\begin{figure}[!ht]
\centering
\includegraphics[width=\textwidth]{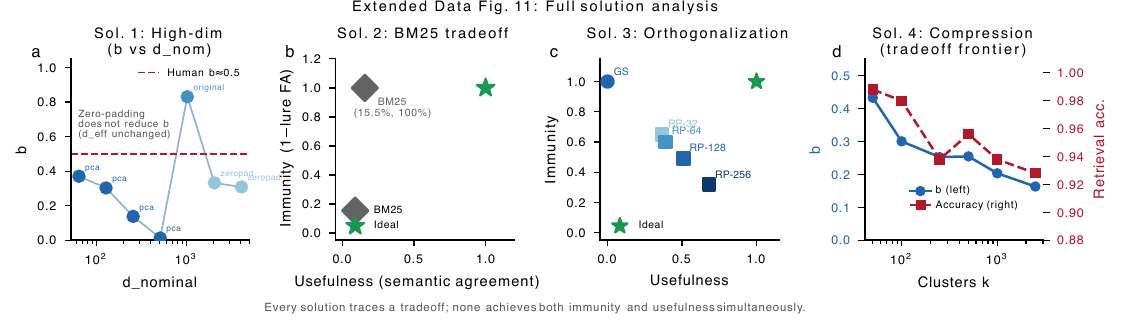}
\caption{\textbf{Extended Data Fig.~10: Full solution analysis.} \textbf{a},~$b$ vs.\ nominal dimensionality. \textbf{b},~BM25 immunity vs.\ usefulness. \textbf{c},~Orthogonalisation methods. \textbf{d},~Compression: $b$ (blue) and accuracy (red) vs.\ cluster count. Every solution traces a tradeoff; none achieves both immunity and usefulness.}
\label{fig:ed_solutions}
\end{figure}

\end{document}